\pdfoutput=1
\documentclass[
    11pt,
    logo,
    onecolumn,
    copyright,
    colorlinks=true,
    allcolors=blue
]{nvidiatechreport}

\usepackage[round]{natbib}
\usepackage{mathtools}
\usepackage{amssymb}
\usepackage{bm}

\usepackage{algorithm}
\usepackage{algpseudocode}

\usepackage{multirow}
\usepackage{makecell}
\usepackage{booktabs}
\usepackage{array}
\usepackage{tabularx}
\usepackage{enumitem}

\usepackage{graphicx}
\usepackage{float}
\usepackage{flafter}
\usepackage{xcolor}
\usepackage{tikz}
\usetikzlibrary{arrows.meta,fit,positioning,calc}
\usepackage{wrapfig}
\usepackage{xspace}

\setlist{nosep,leftmargin=*}

\newcolumntype{L}{>{\raggedright\arraybackslash}X}
\newcolumntype{C}{>{\centering\arraybackslash}X}
\newcolumntype{R}{>{\raggedleft\arraybackslash}X}

\definecolor{activeblue}{HTML}{DCEBFA}
\definecolor{warmorange}{HTML}{FCE4C2}
\definecolor{cachegray}{HTML}{E7E8EA}
\definecolor{commitgreen}{HTML}{DDEEDB}

\newcommand{\OURMETHODtext}{OBA}
\newcommand{\OURMETHODlong}{One Block Ahead\xspace}
\newcommand{\OURMETHOD}{\OURMETHODtext\xspace}

\newcommand{\base}{\textsc{Native}}
\newcommand{\basec}[1]{\base{} ($C{=}#1$)\xspace}
\newcommand{\basesm}{\basec{256}}
\newcommand{\baselg}{\basec{512}}
\newcommand{\ocabidirectional}{\textsc{Dense Mask}\xspace}
\newcommand{\flowblock}{\textsc{FlowBlock}\xspace}
\newcommand{\flowblockw}{\textsc{FlowBlock} ($W{=}2$)\xspace}

\newcommand{\rettoken}{\textsc{Keep Tokens}\xspace}
\newcommand{\retsignal}{\textsc{Keep Signal}\xspace}
\newcommand{\retfull}{\textsc{Keep All}\xspace}
\newcommand{\retboth}{\textsc{Drop Time}\xspace}
\newcommand{\retnosignal}{\textsc{No Signal}\xspace}
\newcommand{\retnoaccept}{\textsc{No Accept}\xspace}

\newcommand{\cur}{\mathrm{c}}
\newcommand{\wu}{\mathrm{w}}

\newcommand{\enc}{\operatorname{Encoder}_{\theta}}
\newcommand{\dec}{\operatorname{Decoder}_{\theta}}
\newcommand{\softmax}{\operatorname{Softmax}}
\newcommand{\unif}{\operatorname{Unif}}
\newcommand{\step}{\operatorname{Step}}
\newcommand{\Vcal}{\mathcal{V}}

\newtheorem{proposition}{Proposition}

\title{One Block Ahead: Promoting Sampler State\\in Diffusion Language Models}


\author[1,$\dagger$]{\mbox{Yangyi Shen}}
\author[2]{\mbox{Pin-Jui Ku}}
\author[2]{\mbox{Nicolai Oswald}}
\author[2]{\mbox{Jean-Marie Lemercier}}
\author[2]{\mbox{Yonggan Fu}}
\author[1]{\mbox{Stefano Ermon}}
\author[1]{\mbox{Jure Leskovec}}
\author[2]{\mbox{Boris Ginsburg}}
\author[2]{\mbox{Ante Juki\'{c}}}
\affil[1]{Stanford University}
\affil[2]{NVIDIA}
\affil[$\dagger$]{Work done during internships at NVIDIA}
\date{}

\hypersetup{
  pdftitle={One Block Ahead: Promoting Sampler State in Diffusion Language Models},
  pdfauthor={Yangyi Shen, Pin-Jui Ku, Nicolai Oswald, Jean-Marie Lemercier, Yonggan Fu, Stefano Ermon, Jure Leskovec, Boris Ginsburg, Ante Juki\'{c}}
}

\begin{document}
\begin{abstract}
\large \textbf{Abstract.}
Block-diffusion language models denoise a whole block of token positions in parallel so each forward pass can emit many tokens. However, blocks are still denoised in strict sequence, a cost that within-block acceleration leaves untouched. Here, we investigate decoding blocks concurrently and propose \OURMETHODlong (\OURMETHODtext), a training-free decoding algorithm that denoises the next block alongside the current one. We demonstrate that for uniform-state diffusion LMs, where a block's denoising progress is not visible in its intermediate tokens, concurrent decoding is advantageous if the next block's tokens are promoted together with the corresponding sampler state, including the self-conditioning signal and the denoising time. We also derive a break-even criterion that relates the benefit of \OURMETHOD to the relative cost of a paired-block forward pass. On DiffusionGemma, a state-of-the-art diffusion LM, \OURMETHOD increases the tokens per forward pass by 42.3\% (from 15.64 to 22.25) and the single-request denoising throughput by 16.3\%, averaged over ten benchmarks, with a small accuracy regression.
\end{abstract}
\maketitle

\begin{figure}[t]
  \centering
  \includegraphics[width=1.0\linewidth]{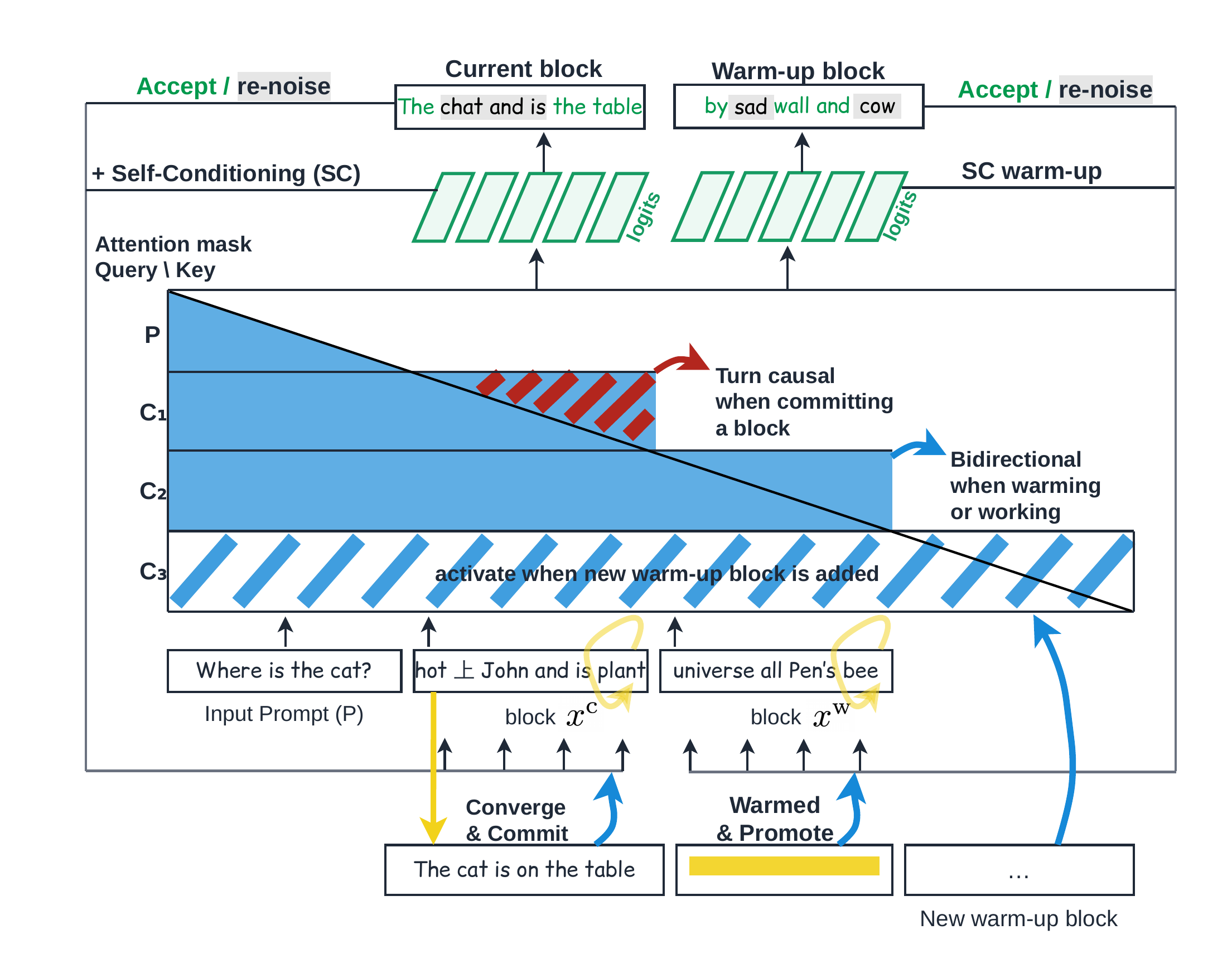}

  \vspace{-0.5cm}\caption{\textbf{\OURMETHODlong (\OURMETHOD).} The current block $x^\cur$ and the warm-up block $x^\wu$ are denoised in one forward pass over the input sequence $[x^\cur\mid x^\wu]$. A block-causal attention mask lets $x^\wu$ attend to $x^\cur$ but not the reverse, so $x^\cur$ is denoised as if it were alone. When $x^\cur$ is committed and encoded under causal attention, $x^\wu$ is promoted in its place, retaining the tokens and sampler state it has accumulated, and a new warm-up block follows.}
  \label{fig:pipeline}
\end{figure}

\section{Introduction}
\label{sec:introduction}

Diffusion language models generate text by iteratively denoising a whole sequence of token positions, i.e., a block, so a single forward pass can decode many positions at once \citep{MDLM,LLaDA,Dream,NLD}. That decoding parallelism is the principal advantage of diffusion over autoregressive decoding, generating one token at a time. However, there is typically a tradeoff with quality: the more positions a step accepts, the more it disregards the dependencies among them, so accepting tokens aggressively improves the latency at the cost of quality \citep{EntropyBoundedSampling,NLD}. How far a decoder can push that quality--throughput frontier is what determines whether diffusion is competitive with autoregression on a single request.

Among recent diffusion language models, DiffusionGemma offers a particularly strong quality--throughput trade-off, reporting around 20 tokens per forward pass (TPF) and roughly 1,500 output tokens per second (TPS) on a single H100 GPU while maintaining strong benchmark performance \citep{DiffusionGemma}.
Its efficiency comes from several design choices. A block-autoregressive architecture supports KV caching and variable-length generation \citep{BlockDiffusion};
uniform-state diffusion \citep{MultinomialDiffusion,UDLM} keeps every position revisable until its block is committed, allowing aggressive parallel updates without the irreversible early commitments of masked diffusion decoding;
and entropy-bounded sampling \citep{EntropyBoundedSampling}, adaptive stopping, and sampler distillation reduce the required denoising steps, with self-conditioning carrying predictions from one step to the next.
However, its decoding remains sequential across blocks, which leaves multi-block parallelism unexplored: can decoding successive blocks concurrently further accelerate a model that already achieves such high within-block efficiency?

Recent work has demonstrated the potential of multi-block decoding through distillation, post-training, and concurrent decoding schedules \citep{D2F,MultiBlockDiffusion,FlowBlock}.
However, whether these gains extend to DiffusionGemma's combination of native 256-token blocks, high decoding efficiency, and strong model quality remains unestablished.
For example, Multi-Block Diffusion~\citep{MultiBlockDiffusion} and FlowBlock~\citep{FlowBlock} report their main results with 32-token blocks and baselines averaging 3.47 and 4.62 TPF, respectively.
Extending concurrency to DiffusionGemma poses both an algorithmic and a hardware utilization challenge.
Its self-conditioning signal and denoising time carry progress that must survive the transition between blocks.
Meanwhile, overlapping two blocks expands a forward pass from 256 to 512 positions.
Small blocks can exploit unused GPU compute capacity, but that headroom diminishes as wider blocks approach the compute-bound regime \citep{FastdLLMv2}.
Consequently, reducing forward-pass counts does not by itself guarantee higher tokens per second.

In this work, we propose \OURMETHODlong{} (\OURMETHODtext), a new training-free multi-block decoding algorithm built upon DiffusionGemma (Figure~\ref{fig:pipeline}). To the best of our knowledge, this is the first study of concurrent multi-block decoding for a diffusion language model that already achieves around 20 tokens per forward pass with native 256-token blocks. We investigate the design choices impacting the effectiveness of concurrent decoding: how to structure attention between the two blocks (current block and warm-up block), and which components of the warm-up state must be retained when the warm-up block becomes the current block (e.g., tokens, self-conditioning signals, and the denoising time). To compare fairly how gains in tokens per forward pass transfer to throughput, we optimize the precision, attention, mixture-of-experts and sampler components of a denoising step for both the single- and the paired-block input sequence layout (Section~\ref{subsec:exp-kernels}).

Averaged over ten benchmarks on DiffusionGemma, \OURMETHODtext{} increases TPF from 15.64 to 22.25 (+42.3\%) and single-request denoising throughput from 2,375 to 2,761 TPS (+16.3\%) on one GB300 GPU, with a small reduction in average accuracy (Section~\ref{sec:experiments}). By contrast, the direct way to add parallelism, doubling the block to 512 positions, lowers throughput by 41.4\% and reduces accuracy further. Our ablations identify sampler-state promotion as central to the gain: withholding self-conditioning during warm-up lowers throughput 19.9\% below the baseline. Together, these results demonstrate that concurrent decoding, supported by sampler-state promotion, can improve throughput even when the baseline already achieves substantial parallelism within each block.

\section{Background}
\label{sec:background}

\subsection{Discrete Diffusion Language Models}
\label{sec:bg-blockar}

A discrete diffusion language model generates text by reversing a corruption process on token sequences rather than by extending a sequence left to right
\citep{D3PM,MultinomialDiffusion}. A forward process maps a clean sequence $x_0$ to progressively noisier $x_t$, $t\in[0,1]$, and a learned denoiser is trained to invert it, so each reverse step may revise every position at once. Two corruption processes are in use. \emph{Masked} diffusion replaces a token with a dedicated mask token and reveals positions monotonically \citep{MDLM,LLaDA,Dream,NLD}; \emph{uniform-state} diffusion replaces it with another token drawn from the vocabulary, so every position holds real text and remains revisable until it is committed \citep{MultinomialDiffusion,SEDD,UDLM,Sumi}. DiffusionGemma is uniform-state \citep{DiffusionGemma}, and the distinction is what makes state inheritance a question in Section~\ref{sec:method-promotion}.

Fully bidirectional diffusion language models cannot reuse a KV cache or vary
their output length, since every position attends to every other. Block
Diffusion removes both limits by factorizing generation over blocks, with
causal attention across blocks and bidirectional attention within the current block
\citep{BlockDiffusion}.
DiffusionGemma adopts this structure and calls each block a \emph{canvas}; to align with the broader diffusion language model literature, we use \emph{block} throughout for this fixed-size group of consecutive positions.
Specifically, a causal encoder $\enc$ turns the context $h$ into a KV cache
$H=\enc(h)$, and a bidirectional decoder refines a block of length $C=256$
while cross-attending to $H$. Once a block is fully denoised, written
$\hat x_0$, one encoder pass appends its keys and values,
$H \gets H \oplus \enc(\hat x_0)$, and only then is the next block initialized;
committed tokens are therefore encoded once and never re-encoded.

\subsection{DiffusionGemma}
\label{sec:bg-block}

DiffusionGemma is the model we build on: an instruction-tuned mixture-of-experts
diffusion language model with 26 billion total and 4 billion active parameters
per token \citep{DiffusionGemma}. Of the design choices behind its efficiency,
three interact with \OURMETHOD, and we specify them here: the entropy-bounded
sampler, adaptive stopping, and the self-conditioning signal.

A block is refined over at most $N=48$ denoising steps with step size
$\Delta t=1/N$. Let $\Vcal$ be the vocabulary, so that $\Vcal^C$ is the set of
token sequences of length $C$, and let $x_t\in\Vcal^C$ be the block at
time $t\in[0,1]$, starting at $t=1$ with $x_1\sim\unif(\Vcal^C)$. Each step
calls the decoder on the block, the cache $H$ and a self-conditioning signal
$z_t$ carrying the previous predictions, returning $L_t=\dec(x_t,z_t,t,H)$, from
which
\begin{equation}
  \hat p_0 = \softmax\big(L_t/\tau_t\big), \qquad
  z_{t-\Delta t} = \operatorname{FFW}(\hat p_0 E), \qquad
  x_{t-\Delta t} \sim \step(x_t,\hat p_0),
\label{eq:denoising-iteration}
\end{equation}
with $E$ the embedding matrix, $\operatorname{FFW}$ the self-conditioning module,
and $\tau_t$ falling linearly from $0.8$ to $0.4$.
The sampler acts in $\step$, deciding position by position which predictions to
keep. Positions are accepted in order of increasing predictive entropy until
their cumulative entropy reaches a budget $b=0.1$, which bounds the joint mutual
information and so keeps them approximately independent
\citep{EntropyBoundedSampling}. An accepted position takes a token sampled from
$\hat p_0$, and every other position is corrupted again by sampling a new token
uniformly from $\Vcal$ \citep{MultinomialDiffusion}.
Denoising stops early once mean entropy falls below
$e_{\mathrm{stop}}=0.005$ with consecutive steps agreeing on the argmax.
Acceptance therefore binds for one step only, and no position is fixed while its
block is still being refined; that ends at commitment, after which the tokens
are committed permanently.

\subsection{Inference Efficiency Metrics}
\label{sec:bg-metrics}

Adaptive stopping makes the number of denoising steps a random variable, so DiffusionGemma measures denoising efficiency empirically \citep{DiffusionGemma}. Let $K$ be the number of blocks a response spans, $C_k$ the valid tokens produced in block $k$, counted up to the first end-of-sequence token, $F$ the denoising passes executed, and $T_{\text{den}}$ the average GPU time of a denoising pass. Tokens per forward and tokens per second are
\begin{equation}
 \text{TPF} \triangleq \frac{\sum_{k=1}^{K}C_k}{1+(K-1)+F},
 \qquad
 \text{TPS} \triangleq \frac{\sum_{k=1}^{K}C_k}{F\,T_{\text{den}}},
 \label{eq:metrics}
\end{equation}
where the TPF denominator counts every forward pass that serves the request: the prompt prefill, the $K-1$ passes that append each committed block to the cache, and the $F$ denoising passes. TPF counts useful output per invocation and is independent of hardware. Decode TPS (throughput) counts useful output per second of denoising GPU time; it excludes the prefill and cache-append passes and host-side overhead, so that it isolates the cost of denoising (Appendix~\ref{app:sweeps} gives wall TPS).

\section{Method}
\label{sec:method}

\OURMETHODlong (\OURMETHOD) denoises two blocks in each forward pass. The decoder runs over the input sequence $[x^\cur\mid x^\wu]$ of $2C$ positions, where $x^\cur$ is the block on its way to being committed and $x^\wu$ is the one after it. As depicted in Figure~\ref{fig:pipeline}, a block-causal mask lets $x^\wu$ attend to $x^\cur$ but not the reverse. When the current block $x^\cur$ is committed, the warm-up block $x^\wu$ is promoted into its place and a new block fills the vacated half. In the following we present design choices for \OURMETHOD: Section~\ref{sec:method-paired} specifies the paired-block layout, Section~\ref{sec:method-promotion} specifies what is inherited at promotion, and Section~\ref{sec:method-amortization} gives the condition under which the \OURMETHOD schedule raises throughput.

\subsection{Paired-Block Layout}
\label{sec:method-paired}

Rather than running the decoder over the $C$ positions of one block, we lay the two blocks side by side and run it over $2C$ positions, so one pass returns logits for both blocks. A single block still commits $C$ tokens and the cache is still appended once per block, so \OURMETHOD composes with any method acting inside a block. The block-causal mask defines the attention pattern of each block: queries in $x^\cur$ attend to the committed cache $H$ and to $x^\cur$ itself, while queries in $x^\wu$ attend to $H$, $x^\cur$ and $x^\wu$, and attention stays bidirectional inside each half as when a block is denoised alone. No query in $x^\cur$ reads a key in $x^\wu$, and every remaining operation acts on one position at a time, so in exact arithmetic the logits over $x^\cur$ are those that block would have produced alone (cf. Appendix~\ref{app:proofs}).

Although the attention mask prevents the current block from attending to the warm-up block, it does not guarantee numerical equivalence to decoding the current block alone. Processing both blocks in a single grouped mixture-of-experts call changes the computation shape and can introduce floating-point differences that alter subsequent sampling decisions, so we compare configurations on matched benchmarks and seeds rather than asserting identical trajectories; per-benchmark accuracy is given in Appendix~\ref{app:sweeps}. Appendix~\ref{app:implementation} describes the paired-block adaptations, including the shared quantization scales computed from the current block and used by both blocks.

\subsection{State Inheritance at Promotion}
\label{sec:method-promotion}

Both blocks take a denoising step from every forward pass, so the warm-up block accumulates progress while the current block is still changing. Which part of that progress the warm-up block keeps at promotion, i.e., when the current block commits, is what we call \emph{state inheritance}. We denote the state of a block after its $n$-th pass as $s_n=(x_n,\, z_n,\, t_n,\, \sigma_n)$, where $x_n$ denotes the block tokens, $z_n$ the self-conditioning signal, $t_n$ the denoising time, and $\sigma_n$ the adaptive stopping bookkeeping (i.e. the recent argmax predictions and the stability counter). We call $(z,t,\sigma)$ the \emph{sampler state}. As DiffusionGemma uses uniform-state diffusion, every position holds a real token from the vocabulary from the first step onward, so $x$ alone does not reveal how far a block has progressed and the sampler state cannot be reconstructed solely from $x$.

Assume the current block is committed after $n_k$ paired denoising passes. Its $C$ tokens are appended to the cache and the warm-up block moves into the current position carrying $(x_{n_k},\, z_{n_k},\, t_{n_k},\, \sigma_{\mathrm{init}})$. The tokens, the signal and the denoising time are promoted, and only $\sigma$ is reset to its initial value $\sigma_\mathrm{init}$. The reset of $\sigma$ prevents the promoted block from stopping on a stability judgment formed before its context has been finalized. A new warm-up block then fills the vacated half and the cycle repeats. No warm-up block commits out of order, and the final current block at the end of generation simply discards its warm-up block (cf. Algorithm~\ref{alg:sampling}, Appendix~\ref{app:algorithm}).

\subsection{Break-Even Criterion}
\label{sec:method-amortization}
Although the paired-block layout advances two blocks per pass, it does not halve the passes per block, because a warm-up step conditions on a current block that has not yet been committed.
Let $\bar n$ be the expected number of denoising passes a block needs under \basesm{}, and $n_k$ the number of passes block $k$ spends as the current block under a concurrent schedule. We model the passes block $k+1$ then needs as the current block as $\bar n-\alpha n_k$, so that $\alpha$ is the expected number of current-block passes saved per pass its predecessor spends as the current block. Under \OURMETHOD, block $k+1$ takes one warm-up step on each of those passes, so $\alpha$ is also the saving per warm-up step; a schedule that starts the warm-up block later, such as \flowblockw{} (Section~\ref{sec:forward-result}), takes fewer warm-up steps per predecessor pass and therefore has a smaller $\alpha$ even if each step is worth as much. We make two simplifying assumptions: (1) $\alpha$ is a constant, independent of $n_k$ and of $k$, i.e., the saving is linear in the predecessor's passes; and (2) $\bar n$ does not depend on $k$. Assumption (1) is an approximation rather than a property of the sampler; we fit $\alpha$ under it from aggregate counters (cf. Appendix~\ref{app:alpha-estimation}) and do not test it at the level of individual blocks. We further expect $0\le\alpha<1$: a warm-up step should not increase the passes a block needs, and it saves less than a full pass because it conditions on a current block that has not yet been committed. Linearity also makes the recurrence exact in expectation, so under these assumptions the expected passes obey the affine recurrence
\begin{equation}
 n_{k+1} = \bar n - \alpha\, n_k,
 \qquad n_1 = \bar n,
 \qquad\text{with unique fixed point}\quad
 n^\star = \frac{\bar n}{1+\alpha} .
 \label{eq:amortization}
\end{equation}
The map $n\mapsto\bar n-\alpha n$ is a contraction of modulus $\alpha<1$, so $n^\star$ exists and is unique. It is also approached geometrically, since $|n_k-n^\star|=\alpha^{\,k-1}|\bar n-n^\star|$.
Let $r$ denote the ratio of $T_{\text{den}}$ in Eq.~\eqref{eq:metrics} under a schedule to that under \basesm{}, measured at batch size one on the same hardware. Since $T_{\text{den}}$ averages over all denoising passes of a request, $r$ is the relative cost of a paired-block pass when every denoising pass is paired, as in \OURMETHOD, and lies between one and that value when a schedule mixes single- and paired-block passes. At the fixed point, denoising passes per block therefore fall by a factor $1+\alpha$ while a denoising pass takes $r$ times as long, so TPS changes by $(1+\alpha)/r$, and the schedule raises throughput when $\alpha>r-1$.

\begin{table}[t]
\centering
\caption{\textbf{Decoding efficiency across ten benchmarks.} Tokens per forward pass (TPF) and tokens per second (TPS). TPS counts denoising GPU time only (Eq.~\eqref{eq:metrics}). All metrics are averaged over three generation seeds. Best in bold. \OURMETHOD achieves higher TPF and TPS than all baselines on all benchmarks. See benchmark accuracy and wall TPS in Appendix~\ref{app:sweeps}.}
\label{tab:main-results}
\small
\setlength{\tabcolsep}{5pt}
\begin{tabularx}{\textwidth}{@{}l*{8}{R}@{}}
\toprule
& \multicolumn{2}{c}{\basesm{}}
& \multicolumn{2}{c}{\baselg{}}
& \multicolumn{2}{c}{\flowblockw}
& \multicolumn{2}{c}{\textsc{\OURMETHOD}} \\
\cmidrule(lr){2-3}\cmidrule(lr){4-5}\cmidrule(lr){6-7}\cmidrule(lr){8-9}
Benchmark
& TPF & TPS
& TPF & TPS
& TPF & TPS
& TPF & TPS \\
\midrule
GSM8K     & 17.59 & 2800 & 19.93 & 2128 & 18.33 & 2816 & \textbf{22.87} & \textbf{2983} \\
MATH-500  & 17.96 & 2798 & 17.52 & 1503 & 18.26 & 2742 & \textbf{23.05} & \textbf{2911} \\
MBPP      & 16.43 & 2485 & 16.64 & 1208 & 16.89 & 2516 & \textbf{22.41} & \textbf{2733} \\
HumanEval & 18.93 & 2963 & 21.98 & 1996 & 22.04 & 3447 & \textbf{29.22} & \textbf{3858} \\
GPQA      & 13.13 & 1919 & 13.92 & 1018 & 14.30 & 2036 & \textbf{20.38} & \textbf{2435} \\
IFEval    & 15.87 & 2433 & 18.51 & 1787 & 17.43 & 2566 & \textbf{23.78} & \textbf{3000} \\
MMLU      & 13.16 & 1997 & 15.72 & 1430 & 14.05 & 2061 & \textbf{20.77} & \textbf{2611} \\
LCB C++   & 14.37 & 2090 & 14.21 & 902  & 14.34 & 2054 & \textbf{19.37} & \textbf{2230} \\
AIME 2024 & 14.59 & 2159 & 14.26 & 1023 & 15.50 & 2247 & \textbf{19.95} & \textbf{2410} \\
AIME 2025 & 14.38 & 2101 & 14.01 & 917  & 16.42 & 2339 & \textbf{20.72} & \textbf{2442} \\
\midrule
Average   & 15.64 & 2375 & 16.67 & 1391 & 16.76 & 2482 & \textbf{22.25} & \textbf{2761} \\
\bottomrule
\end{tabularx}
\end{table}

\section{Experiments}
\label{sec:experiments}
We design the experimental evaluation to understand three different aspects of \OURMETHOD: (1) Does denoising one block ahead raise throughput over single-block and multi-block baselines across diverse tasks, and at what accuracy cost? (2) Which parts of the sampler state carried across promotion make the wider denoising pass pay for itself? (3) How do the attention mask between the two blocks and the time at which the next block starts affect the break-even criterion $\alpha>r-1$?

\subsection{Experimental Setup}
\label{subsec:exp-setup}

\textbf{Model and benchmarks.} We evaluate DiffusionGemma across ten benchmarks on a single NVIDIA GB300 at batch size one: GSM8K, MATH-500, MBPP, HumanEval, GPQA, IFEval, MMLU, LiveCodeBench C++, AIME 2024 and AIME 2025 \citep{GSM8K,MATH,MBPP,HumanEval,GPQA,IFEval,MMLU,LiveCodeBench}. We report the main comparison averaged over three generation seeds; Appendix~\ref{app:protocol} gives the details about the evaluated population, graders, seeds and serving limits. While the proposed paired-block layout applies to any block-diffusion model, uniform-state diffusion is what makes state inheritance at promotion a design question. Existing uniform-state models, e.g., Sumi~\citep{Sumi}, UDLM~\citep{UDLM} and Duo~\citep{Duo}, are trained and sampled over a full sequence in a single diffusion trajectory, so no state has to cross a block commitment boundary. DiffusionGemma is, to our knowledge, the first publicly released large-scale model that is both uniform-state and block-autoregressive.

\textbf{Baselines.} \basesm{} is the default DiffusionGemma inference configuration, denoising one 256-token block at a time. \baselg{} applies the same algorithm to a single 512-token block and commits all 512 positions at once, the natural way to increase parallelism by widening a block.
\flowblockw{} is a concurrent, training-free multi-block decoder built on the LLaDA2.1 token-editing model~\citep{FlowBlock,LLaDA21}; we apply it by replacing its masked diffusion model forward with DiffusionGemma's (cf. Appendix~\ref{app:baselines}). It decodes up to $W$ blocks concurrently, and we use $W{=}2$, the window of its main results. It starts the next block once its readiness gate exceeds a threshold, and since blocks are decoded concurrently, that block keeps its tokens, self-conditioning signal and denoising time when its predecessor commits. It therefore differs from \OURMETHOD mainly in when the next block starts, and we use it as our comparison for the start time of the warm-up block.
We exclude other multi-block methods~\citep{D2F,MultiBlockDiffusion}, which target masked diffusion models that cannot revise a revealed token.
We further isolate \OURMETHOD's mechanism through ablations of state inheritance and attention between the two blocks. For fair comparison, all configurations above run with kernels specialized to their own input sequence shape.

\subsection{Decoding Efficiency}
\label{sec:forward-result}

\textbf{TPF and throughput.} We investigate the decoding efficiency of the proposed \OURMETHOD, with the main results summarized in Table~\ref{tab:main-results}. \OURMETHOD has higher TPF and TPS than all baselines on all ten benchmarks. Averaged over the ten benchmarks, \OURMETHOD increases TPF by 42.3\% over \basesm{} (22.25 vs.\ 15.64) and TPS by 16.3\% (2,761 vs.\ 2,375), even though each denoising pass costs 1.26$\times$ as much. The margin over \basesm{} is largest on MMLU, HumanEval and GPQA, where TPS rises 31\%, 30\% and 27\% and TPF 58\%, 54\% and 55\%, and smallest on MATH-500, where TPS rises 4\%. MMLU and GPQA are the benchmarks where a block needs the most denoising passes under \basesm{}, so the current block occupies the input sequence for longer and the warm-up block accumulates more steps before it is promoted. These gains come with a 1.37-percentage-point reduction in average accuracy over the ten benchmarks (cf. Table~\ref{tab:appx-accuracy}).

\textbf{Comparison with FlowBlock.} \flowblockw{} raises average TPF over \basesm{} by 7.2\% (15.64 to 16.76) and TPS by 4.5\% (2,375 to 2,482), against 42.3\% and 16.3\% for \OURMETHOD. Since it retains the same sampler state, the difference comes from when it starts the next block, which moves both sides of the break-even criterion of Section~\ref{sec:method-amortization}: waiting for the readiness threshold keeps 92.2\% of denoising passes 256 positions wide, lowering $r$ to 1.03, but forfeits most warm-up steps, lowering $\alpha$ to 0.09. On DiffusionGemma, the gated start thus lowers $\alpha$ by more than it lowers $r-1$, so the margin $\alpha-(r-1)$ is 0.06 under \flowblockw{} against 0.28 under \OURMETHOD, which starts the warm-up block at the earliest possible pass.
The readiness gate guards against drafting on an unsettled prefix, but under uniform-state diffusion an early draft is not locked in: even under \basesm{}, only 5.1\% of positions keep their first-pass token until commitment (Figure~\ref{fig:stability}), since every position not accepted is renoised.
The gate does retain 0.73 points of accuracy over \OURMETHOD (83.43 against 82.70), but at the cost of most of the throughput gain. Widening \flowblock{}'s window to $W{=}3$ or $4$ also makes little difference, with TPF and TPS within 1\% of $W{=}2$ (cf. Appendix~\ref{app:baselines}).

\begin{figure}[t]
\centering
\includegraphics[width=\textwidth]{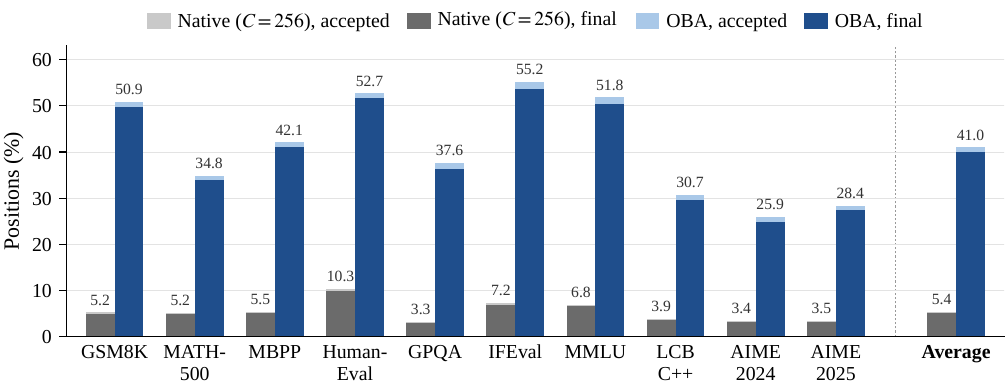}
\caption{\textbf{Warm-up progress of promoted and newly initialized blocks.} For each benchmark, the grey bar corresponds to \basesm{} and the blue bar to \OURMETHOD. The full height of a bar is the fraction of the 256 positions accepted on a block's first pass as the current block, and its dark portion is the fraction whose token at that pass is retained until commitment, so the light segment on top is the fraction of first-pass acceptances later revised. Under \OURMETHOD the first pass is the one after promotion, and every first block, which has no predecessor, is excluded. Fractions are averaged over three generation seeds. On average, a promoted block under \OURMETHOD begins its first pass as the current block with $39.9\%$ of its positions already holding their committed tokens, against $5.1\%$ for a newly initialized block under \basesm{}. Per-benchmark values are given in Table~\ref{tab:stability}.}
\label{fig:stability}
\end{figure}

\textbf{Comparison with a wider block.}
A block-diffusion decoder raises parallelism by committing a wider block, but DiffusionGemma's 256-token block is already close to the point where a single-request forward pass saturates the arithmetic intensity of the GPU \citep{FastdLLMv2}, so the extra positions are no longer nearly free. \baselg{} raises TPF from 15.64 to 16.67, but a denoising pass over 512 positions costs roughly $1.84\times$ one over 256 positions, so its throughput is 41.4\% below \basesm{}. The paired-block layout of \OURMETHOD occupies the same 512 positions at $1.26\times$ the reference cost, largely because its attention is block-causal rather than dense; Section~\ref{sec:ablations} ablates the mask directly. Simply widening the block is therefore not the axis to scale up.

\textbf{Warm-up progress at promotion.} Figure~\ref{fig:stability} measures what a block has settled on its first pass as the current block, which under \OURMETHOD is the first pass after promotion, conditioned on the newly committed context, and under \basesm{} the first pass after initialization from noise. Averaged over the ten benchmarks, $41.0\%$ of the 256 positions of a promoted block are accepted on that pass, against $5.4\%$ for a newly initialized block, and $39.9\%$ against $5.1\%$ keep that pass's token until commitment, so a promoted block arrives with more than a third of its positions already holding their committed tokens. The fraction of first-pass acceptances retained until commitment is similar under both schedules, $97\%$ and $94\%$. The state inherited from warm-up therefore lets many more positions clear the entropy budget on the first pass under the committed context without making those acceptances less reliable, so promotion raises how many positions settle early rather than how reliably they settle. The effect varies with the task. More than half of the positions are accepted on the first pass on IFEval, HumanEval, MMLU and GSM8K, against $26$--$31\%$ on AIME and LiveCodeBench C++, where a block plausibly depends more on the exact content of its predecessor.

\subsection{Ablation Study}
\label{sec:ablations}

\begin{table}[t]
\centering
\small
\caption{\textbf{Inheritance policies.} $\alpha$, relative pass cost $r$, efficiency, and accuracy for seven policies, averaged over the ten benchmarks. Retains gives the inherited state in the $(x,z,t)$ decomposition; Before promotion marks the three policies that instead change what a block does before promotion; $\Delta$TPS is against \basesm{}.}
\label{tab:promotion-ablation}
\setlength{\tabcolsep}{4pt}
\begin{tabular}{@{}lllrrrrrr@{}}
\toprule
Policy & Retains & Before promotion & $\alpha$ & $r$ & TPF & TPS & $\Delta$TPS \% & Acc. \\
\midrule
\basesm{} & -- & -- & 0 & 1.00 & 15.64 & 2375 & -- & 84.07 \\
\midrule
\retfull{} \textsc{(\OURMETHOD)} & $x,z,t$ & full & 0.54 & 1.26 & 22.25 & 2761 & $+16.3$ & 82.70 \\
\retboth{} & $x,z$ & full & 0.41 & 1.26 & 20.76 & 2558 & $+7.7$ & 82.69 \\
\rettoken{} & $x$ & full & 0.25 & 1.26 & 18.83 & 2309 & $-2.8$ & 83.55 \\
\retsignal{} & $z$ & full & 0.32 & 1.27 & 19.61 & 2401 & $+1.1$ & 83.93 \\
\retnoaccept{} & $z$ & no acceptance & 0.32 & 1.27 & 19.62 & 2403 & $+1.2$ & 84.95 \\
\retnosignal{} & $x$ & no self-conditioning & 0.02 & 1.27 & 15.87 & 1902 & $-19.9$ & 84.11 \\
\ocabidirectional{} & $x,z,t$ & dense mask & 0.38 & 1.70 & 20.05 & 1878 & $-20.9$ & 80.29 \\
\bottomrule
\end{tabular}
\end{table}

We ablate what a warm-up block inherits and how it is formed in Table~\ref{tab:promotion-ablation}. \retfull{} retains the full state and is \OURMETHOD; \retboth{} drops the denoising time, \retsignal{} keeps the self-conditioning signal alone, and \rettoken{} keeps the tokens alone. Up to its start threshold, \rettoken{} matches the original FlowBlock~\citep{FlowBlock}, where a block carries only its tokens across the commitment boundary, as under masked diffusion. Three further policies change the warm-up process rather than the inherited state: \retnoaccept{} keeps the self-conditioning signal but accepts no position before promotion, \retnosignal{} accepts positions but withholds self-conditioning, and \ocabidirectional{} replaces the block-causal mask with full bidirectional attention between the two blocks. We estimate $\alpha$ for each configuration by fitting the recurrence of Eq.~\eqref{eq:amortization} on each benchmark to its numbers of denoising passes and blocks and averaging over benchmarks; Appendix~\ref{app:alpha-estimation} gives the details. Since the cache grows with each response, $r$ also depends on response length; the six paired-layout policies have $r$ between 1.26 and 1.27.

\textbf{Inherited state.} Recall that $\alpha$ is the number of current-block passes saved per warm-up step. Promoting the tokens alone, as a masked diffusion decoder would, gives $\alpha=0.25$, below the break-even value $r-1=0.26$, and changes TPS by $-2.8\%$ relative to \basesm{}. Adding the self-conditioning signal raises $\alpha$ to 0.41 and TPS by $7.7\%$, and adding the time raises $\alpha$ to 0.54 and TPS by $16.3\%$, since a block that restarts its temperature schedule spends passes re-deciding positions it had already settled. Each component of the sampler state therefore adds to the saving per warm-up step, and together they account for the entire throughput gain. Carried alone, the signal is worth more than the tokens: \retsignal{} reaches $\alpha=0.32$ against $0.25$ for \rettoken{}.

\textbf{Warm-up process.} \retnoaccept{} and \retsignal{} both reach $\alpha=0.32$, with TPS changing by $+1.2\%$ and $+1.1\%$ relative to \basesm{}. Once the self-conditioning signal is propagated during warm-up, whether some positions were allowed to settle makes little difference. On the other hand, \retnosignal{} lowers $\alpha$ to $0.02$, and changes TPS by $-19.9\%$ relative to \basesm{}. Without self-conditioning, the warm-up does no useful work while the schedule still pays for the wider pass, falling below the break-even criterion of Section~\ref{sec:method-amortization}.

\textbf{Attention between blocks.} \ocabidirectional{} keeps the promoted state and the commitment cadence of \OURMETHOD but lets each block attend to the other, which removes the guarantee of Proposition~\ref{prop:isolation}. It still raises TPF over \basesm{}, 20.05 against 15.64, at $\alpha=0.38$, so the warm-up block does useful work under mutual attention, though less than under the block-causal mask.
However, a pass under \ocabidirectional{} costs $1.70\times$ against $1.26\times$ for the block-causal one, so break-even would require $\alpha>0.70$, and throughput falls 20.9\% below \basesm{}. The gap to the $1.84\times$ of \baselg{}, which shares its input sequence shape and kernels, reflects the context lengths of their responses rather than the mask.
It also lowers average accuracy relative to \OURMETHOD, 80.29 against 82.70, since the current block now conditions on a block that is still near noise.

\subsection{Kernel Implementation}
\label{subsec:exp-kernels}

\textbf{Shared kernels.} Every configuration runs the same implementation, which we optimize for the GB300 in each component of a denoising step (Eq.~\eqref{eq:denoising-iteration}). In precision, the mixture-of-experts layers use a grouped matrix multiplication with 4-bit weights and activations (NVFP4) in place of 4-bit weights with 16-bit activations, the dense projections use FP8, the rotary tables are kept in FP32, and amax reduction, scale finalization and row quantization are fused into one pass. In attention, the kernels are fused and additionally parallelized over the cache $H$, the ragged masked softmax of the fallback path is vectorized, and the sliding-window keys and values are held in a ring buffer. In the mixture-of-experts and dense layers, the input producer is fused, the combine stage is sharded over the hidden dimension, the routing scan is parallelized in a deterministic accumulation order, and the two branches of a layer execute concurrently on separate streams. In the sampler, the self-conditioning module $\operatorname{FFW}$ is vectorized. Table~\ref{tab:main-results} therefore compares decoding algorithms at a common level of optimization, and the gains are over an already optimized single-block decoder.

\textbf{Per-configuration kernels.} Each configuration additionally selects the kernels specialized to its own input sequence shape. \basesm{} uses the 256-position attention kernels; \baselg{} and \ocabidirectional{} use matrix-multiplication and mixture-of-experts plans tuned for 512 positions, together with a dispatch correction without which those attention kernels are disabled at that length; \OURMETHOD{} adds offset-aware attention that skips the masked-out half of the paired-block layout and a fused producer over both halves, neither of which applies to a dense input sequence; \flowblockw{} uses the \basesm{} kernels for single-block passes and the \OURMETHOD kernels for paired ones. The one numerical difference across configurations is that \OURMETHOD{} computes its activation-scale from the current block and shares it with the warm-up half.
A denoising pass takes 7.08\,ms under \basesm{}, 8.94\,ms under \OURMETHOD{} and 13.06\,ms under \baselg{}; averaged over benchmarks, a paired pass costs $r=1.26$ against $r=1.84$ for a dense pass over the same 512 positions. Appendix~\ref{app:implementation} gives the full details.

\section{Related Work}
\label{sec:related}

\textbf{Within-block acceleration.} One line of work reduces the number of denoising passes a single block requires. Fast-dLLM \citep{FastdLLM} pairs a block-wise approximate KV cache with confidence-thresholded acceptance, entropy-bounded sampling \citep{EntropyBoundedSampling} replaces the per-position threshold with a cumulative-entropy budget that bounds the joint mutual information the acceptance discards, and dParallel \citep{dParallel} distills a model along its own decoding trajectory while forcing predictive certainty, so that fewer steps suffice. Editing reaches the same end by a different route: the token-to-token stream that LLaDA2.1 \citep{LLaDA21} adds to LLaDA2.0 \citep{LLaDA20} overwrites a revealed token once an alternative becomes sufficiently confident, which permits a lower acceptance threshold because a premature decision can be repaired rather than carried into the output. Uniform-state diffusion has that property natively, since every position holds a real token and stays revisable until its block commits. All of these act within the block being denoised, whereas \OURMETHOD changes how many blocks a pass advances; DiffusionGemma already carries several of them, so \OURMETHOD starts from a baseline at around 20 TPF \citep{DiffusionGemma} rather than the low single digits at which prior concurrent decoding methods are evaluated.

\textbf{Multi-block decoding.} Another line of work focuses on serialization across blocks. D2F \citep{D2F} and MBD-LM \citep{MultiBlockDiffusion} post-train a masked diffusion LM to denoise a block from an incomplete prefix and decode a bounded window of active blocks; FlowBlock \citep{FlowBlock} does so without training, admitting a downstream block once most of its predecessor's remaining masked positions are confident enough to be unmasked, and relying on token-to-token editing to repair tokens drafted against a stale prefix. Three properties separate this line from \OURMETHOD. First, each targets masked diffusion, where the mask ratio reports a block's progress and a promoted block carries nothing beyond its revealed tokens; under uniform state progress is invisible in $x$, and the sampler state $(z,t,\sigma)$ must cross the commitment boundary (Section~\ref{sec:method-promotion}). Second, all use block sizes close to 32, where the extra positions are close to free; at DiffusionGemma's 256-token block a paired-block pass costs $r=1.26$, so the schedule pays only when $\alpha>r-1$ (Section~\ref{sec:method-amortization}). Third, \OURMETHOD is training-free and needs no threshold to decide when to start the warm-up block; applied to DiffusionGemma with the same sampler state, FlowBlock's threshold delays the next block enough that its throughput gain over \basesm{} is 4.5\%, against 16.3\% for \OURMETHOD (Section~\ref{sec:forward-result}).

\section{Conclusion}
\label{sec:conclusion}
We presented \OURMETHODlong{} (\OURMETHODtext), a training-free decoding algorithm that denoises the next block alongside the current one in a single forward pass. A block-causal mask leaves the current block denoised as if it were alone, and at commitment the warm-up block is promoted with its tokens and sampler state. Such a schedule raises throughput when the passes saved per warm-up step, $\alpha$, exceed the extra cost of a paired-block pass, $r-1$. On DiffusionGemma, averaged over ten benchmarks, \OURMETHOD meets this condition with $\alpha=0.54$ against $r-1=0.26$, raising tokens per forward pass (TPF) from 15.64 to 22.25 and single-request denoising throughput (TPS) by 16.3\% at a cost of 1.37 accuracy points. Ablations attribute this gain to the sampler state: promoting the tokens alone gives $\alpha=0.25$ and lowers throughput, and adding the self-conditioning signal and then the denoising time raises $\alpha$ to 0.41 and 0.54. More broadly, these results show that parallelism across blocks remains available to a model that already decodes one block efficiently, and that it is unlocked by promoting the sampler state across the commitment boundary.

\section{Limitations}
\label{sec:limitations}

We evaluate at batch size one, where the latency advantage of diffusion over autoregressive decoding is most pronounced, and on a single GPU type. We evaluate a single model, as DiffusionGemma is to our knowledge the only diffusion language model that is both uniform-state and block-autoregressive. The promotion mechanism is specific to that setting: under masked diffusion a block's progress is readable from its mask and its revealed tokens are the only state to inherit.

\subsection*{AI use statement}
In this work, we used generative AI coding agents (Claude Code and Codex) to implement methods: the \OURMETHOD decoding algorithm, the baseline configurations including the \flowblock adaptation, the inference kernels described in Appendix~\ref{app:implementation}, and the evaluation pipeline, and to run the experiments. We have not used generative AI tools to propose research ideas or hypotheses, to develop the conceptual framework, to design the methodology or experiments, to formulate the mathematical claims or write the proofs in Section~\ref{sec:method-amortization} and Appendix~\ref{app:proofs}, or to interpret results; these were done by the authors. Generating synthetic data, cleaning or reformatting datasets, translation, and qualitative data analysis are not applicable to this work. Additionally, we used generative AI tools to edit parts of the manuscript for readability. We have reviewed all AI-assisted work, and every AI-assisted component was verified by at least two authors: the generated code was reviewed, the kernel-level changes were validated against reference implementations for numerical agreement and unchanged accuracy, the reported numbers were checked against the per-case run records, and every claim, equation, and citation in the manuscript was checked. We take responsibility for the final content of this work, including text, claims, code, and experimental results produced with the aid of generative AI.

\subsection*{Ethics statement}

This work adheres to the ICLR Code of Ethics. All experiments use publicly available benchmarks and a pretrained model (DiffusionGemma), with our contributions centered on methodological and algorithmic advancements, in particular the development of the \OURMETHODlong{} (\OURMETHODtext) decoding algorithm. The method is training-free and introduces no new capabilities beyond those of the underlying model. The study does not involve human subjects, personal data, or other sensitive information, and no applications with a high likelihood of causing harm are considered. We have carefully examined potential risks and broader societal impacts of this research and did not identify significant ethical concerns.

\subsection*{Reproducibility statement}

\OURMETHOD is training-free and requires no change to model weights. The decoding algorithm is specified in Section~\ref{sec:method} and in full as Algorithm~\ref{alg:sampling} in Appendix~\ref{app:algorithm}; the attention mask, the promoted state, and the reset at promotion are given in Sections~\ref{sec:method-paired} and~\ref{sec:method-promotion}. The assumptions behind the break-even criterion are stated in Section~\ref{sec:method-amortization}. Section~\ref{subsec:exp-setup} lists the model, the ten benchmarks, the generation seeds, the hardware, and the baseline configurations, and all sampler hyperparameters are the DiffusionGemma defaults reported in Section~\ref{sec:bg-block}. Appendix~\ref{app:protocol} gives the seed offsets and serving limits, Appendix~\ref{app:baselines} gives the \flowblock thresholds, and Appendix~\ref{app:implementation} describes the kernel-level adaptations needed to reproduce the reported throughput, since it, unlike TPF, depends on the implementation.

\bibliography{one-block-ahead-white-paper}
\bibliographystyle{references}

\clearpage
\appendix

\section{Complete Decoder State Machine}
\label{app:algorithm}

Algorithm~\ref{alg:sampling} gives the full decoding loop, with the additions
to the \basesm{} sampler marked in blue. The state machine below describes the
same loop in words, for $C=256$.

\begin{algorithm}[H]
\caption{DiffusionGemma sampling with
\textcolor{blue}{one-block-ahead decoding}}
\label{alg:sampling}
\footnotesize
\begin{algorithmic}[1]

\Require Context $h$; block length $C=256$;
     maximum denoising steps $N=48$
\Require Entropy budget $b=0.1$;
     stopping threshold $e_{\mathrm{stop}}=0.005$
\Require Temperature endpoints
     $\tau_{\max}=0.8$, $\tau_{\min}=0.4$;
     output stopping rule
\Require \textcolor{blue}{
  Warm-up block $x^\wu$ and paired input sequence $[x^\cur\mid x^\wu]$
}

\Function{DenoiseStep}{$L,t$}
  \State $\tau\gets
    (\tau_{\max}-\tau_{\min})t+\tau_{\min}$,
    \quad $\hat p_0\gets\softmax(L/\tau)$
  \State $\hat x_0\gets\arg\max L$,
    \quad $e^i\gets\operatorname{Entropy}(\hat p_0^i)$
  \State $U\gets
    \operatorname{LowEntropySet}((e^i)_{i=1}^{C},b)$
    \Comment{Entropy-bounded accept set}
  \State Draw $x_D^i\sim\hat p_0^i$ and
    $\tilde x^i\sim\unif(\Vcal)$
  \State $x_U^+\gets x_{D,U}$,
    \quad $x_{\bar U}^+\gets\tilde x_{\bar U}$,
    \quad $z^+\gets\operatorname{FFW}(\hat p_0E)$
  \State \Return $(\hat x_0,e,x^+,z^+)$
\EndFunction

\State $H\gets\enc(h)$,
    \quad $\Delta t\gets1/N$,
    \quad $Y\gets()$
\State $x^\cur\sim\unif(\Vcal^C)$,
    \quad $z^\cur\gets\boldsymbol{0}$

\State \textcolor{blue}{$
  x^\wu\sim\unif(\Vcal^C),\quad z^\wu\gets\boldsymbol{0}$}
\State \textcolor{blue}{$n_{\mathrm{inh}}\gets0,\quad n_\wu\gets0$}
  \Comment{\textcolor{blue}{
    $n_{\mathrm{inh}}$: steps inherited by $x^\cur$; $n_\wu$: steps accumulated by $x^\wu$
  }}

\Loop
  \State $\hat x_0^{\cur,\mathrm{prev}}\gets\bot$
    \Comment{\textcolor{blue}{Reset $\sigma$ after promotion}}

  \State \textcolor{blue}{$c_\cur\gets\min\{n_{\mathrm{inh}},N-1\}$}
    \Comment{\textcolor{blue}{Reserve one committed-context step}}

  \For{$n=1,\ldots,N-\textcolor{blue}{c_\cur}$}
    \State $t_\cur\gets
      1-\bigl(n-1+\textcolor{blue}{c_\cur}\bigr)\Delta t$
    \State \textcolor{blue}{$t_\wu\gets1-(n-1)\Delta t$}

    \State \textcolor{blue}{$
      (L^\cur,L^\wu)\gets
      \dec(
        [x^\cur\mid x^\wu],
        [z^\cur\mid z^\wu],
        (t_\cur,t_\wu),
        H
      )$
      }
      \Comment{\textcolor{blue}{$x^\cur$ sees $H$ and itself; $x^\wu$ also sees $x^\cur$}}

    \State $(\hat x_0^\cur,e^\cur,x^{\cur,+},z^{\cur,+})
      \gets\Call{DenoiseStep}{L^\cur,t_\cur}$

    \State \textcolor{blue}{$
      (\hat x_0^\wu,e^\wu,x^{\wu,+},z^{\wu,+})
      \gets\Call{DenoiseStep}{L^\wu,t_\wu}$}
      \Comment{\textcolor{blue}{No stopping test for the warm-up block}}

    \State \textcolor{blue}{$
      (x^\wu,z^\wu,n_\wu)
      \gets(x^{\wu,+},z^{\wu,+},n_\wu+1)$}
      \Comment{\textcolor{blue}{Keep the update even if the current block stops}}

    \State $\mathrm{stop}_\cur\gets
      (\hat x_0^{\cur,\mathrm{prev}}\neq\bot)
      \land
      (\operatorname{mean}(e^\cur)\leq e_{\mathrm{stop}})
      \land
      (\hat x_0^\cur=\hat x_0^{\cur,\mathrm{prev}})$

    \State \textbf{if }
      $\mathrm{stop}_\cur
      \lor n=N-\textcolor{blue}{c_\cur}$
      \textbf{ then break}

    \State $(x^\cur,z^\cur,\hat x_0^{\cur,\mathrm{prev}})
      \gets(x^{\cur,+},z^{\cur,+},\hat x_0^\cur)$
  \EndFor

  \State $y\gets\operatorname{Publish}(\hat x_0^\cur)$,
    \quad $Y\gets Y\mathbin{\|}y$
    \Comment{\textcolor{blue}{Commit the current block only}}

  \State \textbf{if output stopping rule is met then return} $Y$
  \State $H\gets H\oplus\enc(y)$

  \State \textcolor{blue}{$
    (x^\cur,z^\cur,n_{\mathrm{inh}})\gets(x^\wu,z^\wu,n_\wu)$}
    \Comment{\textcolor{blue}{Promote $x^\wu$ with its denoising time}}

  \State \textcolor{blue}{$
    x^\wu\sim\unif(\Vcal^C),\quad
    z^\wu\gets\boldsymbol{0},\quad
    n_\wu\gets0$}
\EndLoop

\end{algorithmic}
\end{algorithm}

\begin{enumerate}
 \item Encode the prompt causally to form $H$. Initialize the current block
 $x^\cur$ and the warm-up block $x^\wu$ from independent noise streams, with
 null self-conditioning, initial denoising time, and fresh stopping
 bookkeeping. Hold them as one input sequence $[x^\cur\mid x^\wu]$ of 512
 positions.
 \item On every denoising iteration, evaluate $[x^\cur\mid x^\wu]$ in one
 block-causal forward pass. Queries in $x^\cur$ cannot read keys in $x^\wu$,
 and queries in $x^\wu$ read $H$, $x^\cur$ and $x^\wu$. This is one forward
 pass and one denoising step for each half.
 \item Apply the same token transition and self-conditioning update to both
 halves. Only $x^\cur$ is tested against the stopping rule. The warm-up block
 is not committed, but its tokens, self-conditioning signal and position on the
 denoising time advance.
 \item When $x^\cur$ stops, commit its $C=256$ tokens up to the first
 end-of-sequence token. If that slice terminates the request, discard the
 warm-up half. Otherwise append the committed tokens to $H$ through the
 ordinary causal encoder pass, counted once per block as the $K-1$ term of
 Eq.~\eqref{eq:metrics} and not as a denoising pass.
 \item After the append, move the warm-up half into the current positions,
 preserving its tokens $x^\wu$, its signal $z^\wu$ and its time $t^\wu$, and
 reset $\sigma$, so that the promoted block is reconsidered under the newly
 committed context.
 \item Seed a new 256-position warm-up block in the vacated half and resume
 the same 512-position loop. Promotion never creates a third live block, since
 the old current half has moved into $H$.
\end{enumerate}

A promoted block inherits at most $N-1$ of its $N$ steps, the
$c_\cur\gets\min\{n_{\mathrm{inh}},N-1\}$ line above, so at least one pass
always runs under the committed context. \basesm{} uses the same causal append
after each committed 256-token slice, and so do \ocabidirectional{}, which
differs from \OURMETHOD{} only in its mask. \baselg{} commits and appends 512 tokens at once and has
no midpoint append.

\section{Kernel Implementation}
\label{app:implementation}

This appendix expands Section~\ref{subsec:exp-kernels}. All five configurations of the three-seed comparison run on one inference engine, which we optimize for a single NVIDIA GB300 so that a difference in throughput reflects a difference in decoding algorithm rather than in kernel quality. \ocabidirectional{} and \flowblockw{} run from later builds of that engine that add their decoding loops; each build was qualified by reproducing every generated token, finish reason and forward count of \basesm{}, \OURMETHOD{} and \baselg{} on a set of archived GSM8K cases. All configurations share the model export, sampling parameters, entropy-bounded sampler, stopping rule, batch size and plan source for the dense kernels, and each receives a fresh kernel profile for its own input sequence shape, so no configuration is measured on a code path specialized for another. Appendices~\ref{app:impl-precision}--\ref{app:impl-sampler} describe the shared kernels, organized by the component of a denoising step (Eq.~\eqref{eq:denoising-iteration}) they act on; Appendix~\ref{app:impl-layout} describes what each configuration selects for its own block layout; Appendix~\ref{app:impl-measurement} gives the measurement protocol behind $T_{\text{den}}$ and $r$. We describe each at the level a reimplementation requires rather than by internal kernel name.

\subsection{Precision and Quantization}
\label{app:impl-precision}

The mixture-of-experts layers use a grouped matrix multiplication with 4-bit
weights and activations in the NVFP4 format, replacing an implementation with
4-bit weights and 16-bit activations; routing, gather, activation
quantization, the GELU/GLU stage and the combine stage keep their original
structure, and the combine stage preserves a deterministic accumulation order
over route slots. The dense projections use FP8 and the rotary tables are
stored in FP32, with a load-time check that their storage type matches the type
the compiled kernels expect. Amax reduction, scale finalization and row
quantization are fused into one pass, which removes the small scale-conversion
nodes that otherwise separate them, and consumers are launched with a dependent
launch on the producer edges we validate individually.

\subsection{Attention}
\label{app:impl-attention}

The decoder uses fused attention kernels, with a separate causal entry for
prompt and context prefill, and is additionally parallelized over the cache $H$
so that a single request at batch size one occupies more of the device. The
ragged masked softmax of the fallback path is vectorized, and its scalar
variant is retained for exact cross-variant validation. Keys and values of the
sliding-attention layers are held in a ring buffer, which reduces their
allocation and is what makes two live blocks affordable at the
context lengths we evaluate.

\subsection{Mixture-of-Experts and Dense Layers}
\label{app:impl-moe}

The mixture-of-experts input producer is fused, the combine stage is sharded
over the hidden dimension, the routing bin-prefix scan is parallelized, the
GELU/GLU stage is vectorized, and the gather-and-quantize stage issues 128-bit
loads. The dense and mixture-of-experts branches of a layer are
data-independent until their outputs are combined, so we execute them
concurrently on separate streams; the transformation is bit-exact and improves
throughput in an A/B test. Tuned plans are provided for both
the 256- and the 512-position shapes.

\subsection{Sampler}
\label{app:impl-sampler}

The output pass of the self-conditioning module $\operatorname{FFW}$ is
vectorized, and its specialization is compiled at two blocks per
multiprocessor, which doubles occupancy without spilling to local memory. The
entropy-bounded acceptance rule of Section~\ref{sec:bg-block} is computed in
one reduction, shared by both halves of a paired pass.

\subsection{Kernels per Block Layout}
\label{app:impl-layout}

\basesm{} denoises 256 positions and uses the corresponding attention kernels. \baselg{} and \ocabidirectional{} denoise 512 positions under a dense mask; both use the matrix-multiplication and mixture-of-experts plans tuned for that shape, and both require a dispatch correction, without which the fused attention kernels are disabled whenever the input sequence holds 512 positions and the slower fallback path is taken silently.
Their $T_{\text{den}}$ therefore differ only through the context lengths their responses reach (Appendix~\ref{app:impl-measurement}).
\flowblockw{} switches between a 256-position pass, which uses the \basesm{} kernels, and a paired 512-position pass once its next block has started, which uses the \OURMETHOD{} kernels below. \OURMETHOD{} denoises the same 512 positions under the block-causal mask of Section~\ref{sec:method-paired} and adds three specializations, whose guards require batch size one, a 512-position input sequence, a 256-token commitment, the block-causal layout and the inherited timescale, and which fail closed outside that contract.

\paragraph{Offset-aware paired attention.}
The sliding-attention layers see the input sequence as two 256-query blocks
with different visibility: $x^\cur$ attends to $H$ and to itself, and $x^\wu$
additionally to $x^\cur$. One dispatch consumes the two device offsets and
lengths over their overlapping cached key and value intervals rather than
issuing two calls, and skips the masked-out quadrant. Neither \basesm{} nor
\baselg{} has two visibility views.

\paragraph{Full-sequence producers with current-block quantization scales.}
Producer sequences shared by both halves run over all 512 positions, while
ownership of the dynamic low-precision scales stays with the leading 256
positions, as the commitment contract requires; the per-sequence expert maxima
this path produces feed the descriptor builder for $x^\cur$ directly, avoiding a
second scan of the leading half. The mixture-of-experts operators consequently
process the paired-block as one grouped call rather than two independent
256-position calls.

\paragraph{Block-parallel completion metadata.}
A batch-one completion kernel assigns one thread per committable position,
performs fixed-order end-of-sequence reductions, preserves both halves when
$x^\cur$ completes, and commits only the leading 256 positions.

None of these changes the state transition, the random draws, the stopping rule
or the forward accounting; they change how the input sequence is dispatched.
Two consequences are worth stating. First, a paired pass computes its
activation scale from $x^\cur$ alone and shares it with $x^\wu$, so the scale a
position sees does not depend on the warm-up half; a scale reduced over both
halves would make the logits over $x^\cur$ depend on $x^\wu$ through
quantization, which the attention mask cannot prevent. Second, grouped
execution over 512 rows changes the reduction shape of the matrix
multiplications, so the logits over $x^\cur$ are not bitwise identical to those
of a single-block pass even under that ownership rule, and a floating-point
difference can change which positions the sampler accepts at a step. This is
why the comparison uses matched prompts and seeds rather than asserting
identical trajectories (Proposition~\ref{prop:isolation}). Committed keys and values are written by the
separate encoder pass counted as the $K-1$ term of Eq.~\eqref{eq:metrics},
rather than from the denoising pass itself, which we tested and found
nondeterministic under load.

\subsection{Measurement Protocol}
\label{app:impl-measurement}

We obtain $T_{\text{den}}$ from the main comparison itself rather than from a separate microbenchmark. For each configuration and benchmark, we divide the GPU time spent in denoising passes, summed over the cases and three seeds at batch size one, by the number of denoising passes $\sum_i F_i$ (Appendix~\ref{app:alpha-estimation}), and report the unweighted mean of these per-benchmark values. The relative cost $r$ is the ratio of a configuration's $T_{\text{den}}$ to that of \basesm{} on the same benchmark, averaged over benchmarks, so it need not equal the ratio of the averaged $T_{\text{den}}$. Table~\ref{tab:fwd-cost} reports both, and Tables~\ref{tab:appx-alpha-r} and~\ref{tab:appx-ablation-r} give $r$ per benchmark.
Because $H$ grows with the committed length of each response, $T_{\text{den}}$ averages over the context lengths each configuration produces, not the cost at a single fixed $H$, so a configuration whose responses run longer pays more per pass. This is why the two dense 512-position rows of Table~\ref{tab:fwd-cost} differ even though \baselg{} and \ocabidirectional{} use the same input sequence shape and kernels: \ocabidirectional{} generates shorter responses (27.3M against 30.3M tokens, 85 against 240 reaching the output ceiling), so fewer of its passes attend over a long cache, and its $T_{\text{den}}$ is 8\% lower. The paired layout saves on attention, the one component whose cost the block-causal mask reduces, while the matrix multiplications and the mixture-of-experts layers are unaffected by the mask. This is the main reason $r=1.26$ rather than the 1.84 of a dense pass.

\begin{table}[H]
\centering
\small
\caption{\textbf{Cost of a denoising pass.} $T_{\text{den}}$, the GPU denoising time per denoising pass, and the relative cost $r$ for each block layout on GB300. Both are unweighted means over the ten benchmarks of values pooled over cases and three seeds within each benchmark. A paired pass costs less than a dense pass over the same 512 positions, which sets the value of $r$ used in the break-even criterion $\alpha>r-1$ of Section~\ref{sec:method-amortization}. \ocabidirectional{} shares the shape and kernels of \baselg{}; its lower $T_{\text{den}}$ reflects shorter responses and hence fewer long-context passes.}
\label{tab:fwd-cost}
\setlength{\tabcolsep}{6pt}
\begin{tabular}{@{}llrr@{}}
\toprule
Layout & Mask & $T_{\text{den}}$ (ms) & $r$ \\
\midrule
256 positions (\basesm{}) & --- & 7.08 & 1.00 \\
512 positions (\baselg{}) & dense & 13.06 & 1.84 \\
512 positions (\ocabidirectional{}) & dense & 12.06 & 1.70 \\
512 positions (\OURMETHOD{}) & block-causal & 8.94 & 1.26 \\
256/512 positions (\flowblockw{}) & block-causal & 7.29 & 1.03 \\
\bottomrule
\end{tabular}
\end{table}
\section{Proofs}
\label{app:proofs}

\begin{proposition}[Block-triangular isolation]
\label{prop:isolation}
Consider one denoising pass in exact arithmetic under the mask of Section~\ref{sec:method-paired}, and assume every non-attention operation either acts on one position at a time or reduces over positions of $x^\cur$ only. Let the current block enter the pass with state $(x^\cur,z^\cur,t_\cur)$ and cache $H$, and compare it with a $C$-position \basesm{} pass from the same state and cache. Then (i) the logits over $x^\cur$ are independent of the warm-up block state $(x^\wu,z^\wu,t_\wu)$ and equal those of the \basesm{} pass; and (ii) if, in addition, the sampler acting on $x^\cur$ receives the same random variates as under \basesm{}, the resulting state $(x^{\cur,+},z^{\cur,+})$ and stopping decision are equal as well.
\end{proposition}

\begin{proof}
(i) The input representations of the current-block positions depend only on $(x^\cur,z^\cur,t_\cur)$, so they are identical in the paired and the \basesm{} evaluation. Suppose they remain identical through layer $\ell$. At layer $\ell+1$, a current-block query attends to exactly the \basesm{} key set, $H$ and $x^\cur$, and to no warm-up key, so its attention output is identical. Every remaining operation at that layer acts on one position at a time or reduces over $x^\cur$ only, so it cannot introduce a dependence on $x^\wu$. By induction, the current-block hidden states and logits are identical at every layer and independent of the warm-up state. (ii) $\textsc{DenoiseStep}$ and the stopping test are deterministic functions of $(L^\cur,t_\cur)$, the previous argmax and the random variates. By (i), identical logits and identical variates give identical $(x^{\cur,+},z^{\cur,+})$ and an identical stopping decision.
\end{proof}

The proposition is a structural statement in exact arithmetic. Shape-specific
reductions and the grouped mixture-of-experts call need not be bitwise
identical across the single-block and paired layouts
(Appendix~\ref{app:impl-layout}), so the empirical comparison uses matched
prompts and seeds rather than asserting identical generated trajectories.

\section{Estimation of \texorpdfstring{$\alpha$}{alpha}}
\label{app:alpha-estimation}

This appendix details how the values of $\alpha$ in Tables~\ref{tab:promotion-ablation}, \ref{tab:appx-alpha-r} and~\ref{tab:appx-ablation-alpha} are obtained. We estimate $\alpha$ separately on each benchmark from the aggregate counters of each configuration, summed over the benchmark's cases and the three generation seeds, under the two assumptions of Section~\ref{sec:method-amortization}, and report the unweighted mean over benchmarks.

\paragraph{Denoising passes and blocks.} For one benchmark and configuration, let $R$ be the number of responses, three times the number of cases, $Y$ the generated tokens and $Q$ the serving forward passes. Each response takes one prompt-prefill pass and $K_i-1$ cache-append passes (Eq.~\eqref{eq:metrics}), so with $K=\sum_i K_i$ the number of blocks, the denoising passes are $F=Q-R-(K-R)=Q-K$. We approximate the number of blocks from the generated tokens, counting the partially filled last block of a response as half a block on average:
\begin{equation}
 K\approx\frac{Y}{C}+\frac{R}{2},\qquad F=Q-K,\qquad n=\frac{F}{K},
 \label{eq:alpha-counts}
\end{equation}
with $C=256$, or $C=512$ for \baselg{}. For \basesm{}, every block starts from noise, so its mean passes per block on the same benchmark is the reference $\bar n$ of Eq.~\eqref{eq:amortization}.

\paragraph{Recurrence fit.} Under Eq.~\eqref{eq:amortization} with constant $\alpha$, a response of $K_i$ blocks has $n_k=n^\star+(-\alpha)^{k-1}(\bar n-n^\star)$, and since $\bar n-n^\star=\alpha n^\star$,
\begin{equation}
 \sum_{k=1}^{K_i}n_k=n^\star\Big[K_i+\frac{\alpha\big(1-(-\alpha)^{K_i}\big)}{1+\alpha}\Big].
\end{equation}
Summing over responses and neglecting $(-\alpha)^{K_i}$ gives
\begin{equation}
 F\approx\frac{\bar n}{1+\alpha}\Big[K+\frac{R\,\alpha}{1+\alpha}\Big],
 \label{eq:alpha-fit}
\end{equation}
one equation in $\alpha$, which we solve numerically on $[-\tfrac12,1)$ with $\bar n$, $F$ and $K$ from Eq.~\eqref{eq:alpha-counts}; it has a single root there for every benchmark and configuration. The second term in the bracket accounts for the first block of each response, which has no predecessor and pays $\bar n$, and for the transient that follows it.

\paragraph{Comparison with a mean over blocks.} Taking $\bar\alpha=\bar n/n-1$ with $n$ the mean over all blocks understates $\alpha$, because every first block pays $\bar n$; dropping the first block of each response overstates it, because the second block is promoted with $n_1=\bar n>n^\star$ warm-up steps and so needs fewer than $n^\star$ passes, after which $n_k$ oscillates toward $n^\star$. Eq.~\eqref{eq:alpha-fit} models both effects. Averaged over benchmarks, \OURMETHOD has $\bar\alpha=0.47$ against the fitted $\alpha=0.54$.

\paragraph{Relation to measured throughput.} Per block, a configuration spends $n$ denoising passes plus, on average, one prefill or cache-append pass, so its TPF is proportional to $1/(n+1)$ at a fixed number of tokens per block. This is why the TPF ratio of \OURMETHOD{} to \basesm{}, $1.43$ averaged over benchmarks, falls short of $1+\bar\alpha=1.47$: the paired layout reduces the denoising passes but not the cache-append pass of each block. TPS in Eq.~\eqref{eq:metrics} counts denoising time only, so on each benchmark the TPS ratio of a configuration to \basesm{} is $(\bar n/n)(\rho/r)$ by construction, with $\rho$ the ratio of tokens per block $Y/K$; this is an identity rather than a test, since $r$ is computed from the same pass counts. Averaged over benchmarks it gives $1.17$ for \OURMETHOD{} and $0.79$ for \ocabidirectional{}. The fixed-point prediction $(1+\alpha)/r$ at the averaged $\alpha$ and $r$ of Table~\ref{tab:promotion-ablation}, $1.22$ and $0.81$, is higher in both cases because the first block of each response pays $\bar n$ and gains nothing from promotion, which $\bar\alpha$ counts and the fixed point does not. Under both estimates, \OURMETHOD{} clears the break-even $\alpha>r-1$ and \ocabidirectional{} does not.

\paragraph{Limitations of the estimate.} The fit inherits the two assumptions of Section~\ref{sec:method-amortization}, a constant $\alpha$ and a $\bar n$ that does not depend on the block index. The block count of Eq.~\eqref{eq:alpha-counts} is an approximation, and neglecting $(-\alpha)^{K_i}$ in Eq.~\eqref{eq:alpha-fit} biases the fit on responses of few blocks; both are coarsest on GSM8K and MMLU, whose responses average about three and five blocks.
\section{Evaluation Protocol}
\label{app:protocol}

\paragraph{Population, grading and seeds.}
The evaluated population is 4,736 cases: GSM8K (1,319), MATH-500 (500), MBPP (500), HumanEval (164), GPQA (198), IFEval (541), the first 1,000 MMLU examples, LiveCodeBench C++ (454), AIME 2024 (30) and AIME 2025 (30). This is the configured population, not every example in each original release. Accuracy is the fraction accepted by a benchmark's primary grader. GSM8K, MATH-500, GPQA, MMLU and the two AIME sets use deterministic parsers or verifiers, IFEval uses strict all-instructions grading, and LiveCodeBench, MBPP and HumanEval use execution-based grading in restricted subprocesses, with grader failures and timeouts recorded separately from semantic failures. Each benchmark starts a fresh server with seed stride one and a per-benchmark base seed, and case identities, prompt-token hashes and seeds are matched across configurations. The main comparison, covering \basesm{}, \baselg{}, \flowblockw{}, \ocabidirectional{} and \OURMETHOD{}, repeats every case under three generation seeds, obtained by adding offsets of 0, 10{,}000 and 20{,}000 to the per-case seed, for $3\times4{,}736=14{,}208$ responses per configuration and $71{,}040$ in total. Every seed, including offset zero, is freshly generated, and within a seed a case receives the same seed under every configuration. Accuracy is single-generation accuracy, not pass@3 or majority voting, averaged over the three seeds. Separate noise streams per block keep warm-up initialization from shifting the current-block stream by construction. The serving limits are 4,096 input, 32,768 generated and 37,888 cached tokens, and a case is counted as truncated when it reaches the output or context ceiling. The ablations of Section~\ref{sec:ablations} follow the same protocol with a single generation seed, except \ocabidirectional{}, which is part of the three-seed comparison.

\paragraph{Averaging over benchmarks.}
Every figure reported over the ten benchmarks, including $\alpha$, $r$ and $T_{\text{den}}$, is an unweighted mean of per-benchmark values, so each benchmark counts equally; weighting by case count would let GSM8K and MMLU, 49\% of the cases, dominate. Averaged over the ten benchmarks and three seeds, \OURMETHOD{} scores 82.70 against 84.07 for \basesm{}, a loss of 1.37 points.

\paragraph{Forward accounting.} For case $i$, the serving forward passes are $Q_i=1+(K_i-1)+F_i=F_i+K_i$: one prompt prefill, $K_i-1$ cache appends and $F_i$ denoising passes, where one paired $[x^\cur\mid x^\wu]$ call counts as one pass and \baselg{} has one append per 512 committed tokens. Within a benchmark, TPF divides the generated tokens $\sum_i Y_i$ by $\sum_i Q_i$ and TPS divides them by the GPU time of the $\sum_i F_i$ denoising passes, with sums over cases and seeds; the average over benchmarks is the unweighted mean of these per-benchmark ratios. The split of $\sum_i Q_i$ into denoising passes and blocks used to estimate $\alpha$ is given in Appendix~\ref{app:alpha-estimation}.

\section{FlowBlock Adaptation}
\label{app:baselines}
FlowBlock~\citep{FlowBlock} decodes up to $W$ blocks concurrently on a masked, token-editing model, admitting the next block once its predecessor is ready and repairing tokens drafted against an uncommitted prefix by editing. We keep its window, its admission rule and its in-order commitment, and replace only its model forward with DiffusionGemma's, which supplies the token transition, entropy-bounded acceptance, renoising and self-conditioning. FlowBlock defines the readiness of the frontier block as the fraction of its still-masked positions whose confidence already clears the unmasking threshold, and as one when no masked position remains~\citep[Eq.~(2)]{FlowBlock}. Under uniform-state diffusion no position is masked, so we derive the analogous quantity from DiffusionGemma's sampler: the still-masked set becomes the positions rejected, and hence renoised, at the previous step, and clearing the threshold becomes acceptance by the entropy-bounded sampler at the current step. Readiness is therefore the fraction of the preceding block's positions rejected at the previous step that are newly accepted at the current one, and one when none were rejected. Adhering to the configurations in FlowBlock, the next block is admitted once readiness exceeds 0.6 on the general and mathematical benchmarks and 0.8 on the coding benchmarks like MBPP, HumanEval and LiveCodeBench C++. As in FlowBlock, readiness is computed from the current-step logits, so a newly admitted block runs at least one pass before it can admit a successor~\citep[Algorithm~1]{FlowBlock}. As blocks are decoded concurrently, a block keeps its tokens, self-conditioning signal and denoising time when its predecessor commits, with the same $N-1$ inheritance cap, stopping-bookkeeping reset and post-promotion pass as \OURMETHOD. At $W{=}2$, the setting of FlowBlock's main results, it therefore differs from \OURMETHOD in when it starts the next block and in a pass width that switches between 256 and 512 positions. Widening the window does not help (Table~\ref{tab:flowblock}): the readiness gate admits a successor rarely enough that more than 92\% of denoising passes run a single block at every width, three blocks are live on at most 0.63\% of passes, and average TPF and TPS stay within 1\% of $W{=}2$.

\begin{table}[H]
\centering
\caption{\textbf{\flowblock{} on DiffusionGemma across window widths $W$}, over three generation seeds. Accuracy, TPF and TPS are unweighted means over the ten benchmarks, as in Table~\ref{tab:main-results}. Live blocks gives the share of denoising passes that run one, two, three or four blocks. $r$ is the cost of a denoising pass relative to \basesm{}, as in Table~\ref{tab:fwd-cost}.}
\label{tab:flowblock}
\small
\setlength{\tabcolsep}{5pt}
\begin{tabular}{@{}lrrrrrrrr@{}}
\toprule
& & & & \multicolumn{4}{c}{Live blocks (\%)} & \\
\cmidrule(lr){5-8}
Setting & Acc.\ (\%) & TPF & TPS & 1 & 2 & 3 & 4 & $r$ \\
\midrule
\basesm{} & 84.07 & 15.64 & 2375 & 100.0 & 0.0 & 0.0 & 0.0 & 1.00 \\
$W{=}2$ & 83.43 & 16.76 & 2482 & 92.2 & 7.8 & 0.0 & 0.0 & 1.03 \\
$W{=}3$ & 83.58 & 16.77 & 2475 & 92.1 & 7.2 & 0.6 & 0.0 & 1.03 \\
$W{=}4$ & 83.58 & 16.68 & 2463 & 92.2 & 7.2 & 0.4 & 0.2 & 1.03 \\
\bottomrule
\end{tabular}
\end{table}

\section{Position-Level Stability of Promoted Blocks}
\label{app:stability}

Table~\ref{tab:stability} reports the per-benchmark values plotted in Figure~\ref{fig:stability}. For each benchmark, Accepted is the fraction of the 256 positions accepted on a block's first pass as the current block, and Final is the fraction whose token at that pass is unchanged at commitment. Final is a subset of Accepted, so their ratio is the fraction of first-pass acceptances that survive to commitment: $97\%$ under \OURMETHOD and $94\%$ under \basesm{} on average.

\begin{table}[H]
\centering
\small
\caption{\textbf{Position-level stability of promoted and newly initialized blocks.} Accepted: fraction of the 256 positions accepted on a block's first pass as the current block. Final: fraction whose token at that pass is unchanged at commitment, a subset of Accepted. Fractions are computed over complete blocks and averaged over three generation seeds.}
\label{tab:stability}
\setlength{\tabcolsep}{5pt}
\begin{tabular}{@{}lrrrr@{}}
\toprule
& \multicolumn{2}{c}{\basesm{}} & \multicolumn{2}{c}{\OURMETHOD} \\
\cmidrule(lr){2-3}\cmidrule(lr){4-5}
Benchmark & Accepted (\%) & Final (\%) & Accepted (\%) & Final (\%) \\
\midrule
GSM8K     & 5.2  & 5.0  & 50.9 & 49.8 \\
MATH-500  & 5.2  & 5.0  & 34.8 & 33.9 \\
MBPP      & 5.5  & 5.2  & 42.1 & 41.1 \\
HumanEval & 10.3 & 10.0 & 52.7 & 51.6 \\
GPQA      & 3.3  & 3.0  & 37.6 & 36.3 \\
IFEval    & 7.2  & 6.9  & 55.2 & 53.7 \\
MMLU      & 6.8  & 6.6  & 51.8 & 50.5 \\
LCB C++   & 3.9  & 3.5  & 30.7 & 29.7 \\
AIME 2024 & 3.4  & 3.1  & 25.9 & 24.9 \\
AIME 2025 & 3.5  & 3.2  & 28.4 & 27.4 \\
\midrule
Average   & 5.4  & 5.1  & 41.0 & 39.9 \\
\bottomrule
\end{tabular}
\end{table}

\section{Per-Benchmark Results}
\label{app:sweeps}

Tables~\ref{tab:appx-accuracy} and~\ref{tab:appx-wall} give per-benchmark accuracy and wall TPS for the main comparison, and Tables~\ref{tab:appx-ablation-accuracy} and~\ref{tab:appx-ablation-wall} give the same for the inheritance policies and the attention ablation of Section~\ref{sec:ablations}; per-benchmark TPF and TPS are in Table~\ref{tab:main-results}. Tables~\ref{tab:appx-alpha-r}, \ref{tab:appx-ablation-alpha} and~\ref{tab:appx-ablation-r} give the per-benchmark $\alpha$ (Appendix~\ref{app:alpha-estimation}) and $r$ (Appendix~\ref{app:impl-measurement}); \basesm{} has $\alpha=0$ and $r=1$ by definition. Wall TPS divides the generated tokens by the summed server request wall time, which includes prefill, cache appends and host-side overhead but excludes model loading, idle time between requests and grading. All values are over three generation seeds, except the five inheritance policies other than \retfull{} in Tables~\ref{tab:appx-ablation-accuracy}--\ref{tab:appx-ablation-r}, which use one.

\begin{table}[H]
\centering
\caption{\textbf{Per-benchmark accuracy (\%) for the main comparison.}}
\label{tab:appx-accuracy}
\scriptsize
\setlength{\tabcolsep}{5pt}
\begin{tabular}{@{}lrrrr@{}}
\toprule
Benchmark & \basesm{} & \baselg{} & \flowblockw{} & \OURMETHOD \\
\midrule
GSM8K & 96.34 & 94.90 & 96.03 & 96.06 \\
MATH-500 & 96.13 & 91.87 & 95.80 & 95.20 \\
MBPP & 94.87 & 91.07 & 94.87 & 93.80 \\
HumanEval & 97.36 & 96.75 & 96.75 & 96.75 \\
GPQA & 71.55 & 62.12 & 70.54 & 69.02 \\
IFEval & 94.09 & 89.28 & 92.67 & 91.44 \\
MMLU & 88.13 & 84.97 & 87.83 & 86.90 \\
LCB C++ & 66.67 & 61.60 & 62.04 & 63.44 \\
AIME 2024 & 74.44 & 62.22 & 74.44 & 71.11 \\
AIME 2025 & 61.11 & 52.22 & 63.33 & 63.33 \\
\midrule
Average & 84.07 & 78.70 & 83.43 & 82.70 \\
\bottomrule
\end{tabular}
\end{table}

\begin{table}[H]
\centering
\caption{\textbf{Per-benchmark wall TPS for the main comparison.} Best in bold.}
\label{tab:appx-wall}
\scriptsize
\setlength{\tabcolsep}{5pt}
\begin{tabular}{@{}lrrrr@{}}
\toprule
Benchmark & \basesm{} & \baselg{} & \flowblockw{} & \OURMETHOD \\
\midrule
GSM8K & 2530 & 1971 & 2516 & \textbf{2624} \\
MATH-500 & 2566 & 1427 & 2495 & \textbf{2605} \\
MBPP & 2296 & 1155 & 2301 & \textbf{2454} \\
HumanEval & 2690 & 1865 & 3057 & \textbf{3338} \\
GPQA & 1809 & 980 & 1895 & \textbf{2216} \\
IFEval & 2252 & 1684 & 2349 & \textbf{2683} \\
MMLU & 1862 & 1358 & 1906 & \textbf{2351} \\
LCB C++ & 1956 & 869 & 1899 & \textbf{2026} \\
AIME 2024 & 2026 & 985 & 2081 & \textbf{2199} \\
AIME 2025 & 1974 & 884 & 2147 & \textbf{2213} \\
\midrule
Average & 2196 & 1318 & 2265 & \textbf{2471} \\
\bottomrule
\end{tabular}
\end{table}

\begin{table}[H]
\centering
\caption{\textbf{Per-benchmark accuracy (\%) for the ablations} of Table~\ref{tab:promotion-ablation}. \retfull{} is \OURMETHOD.}
\label{tab:appx-ablation-accuracy}
\scriptsize
\setlength{\tabcolsep}{4pt}
\begin{tabular}{@{}lrrrrrrr@{}}
\toprule
Benchmark & \retfull{} & \retboth{} & \retsignal{} & \retnoaccept{} & \rettoken{} & \retnosignal{} & \ocabidirectional{} \\
\midrule
GSM8K & 96.06 & 96.36 & 96.21 & 95.98 & 95.45 & 96.29 & 95.32 \\
MATH-500 & 95.20 & 95.20 & 96.20 & 95.40 & 96.60 & 97.20 & 93.20 \\
MBPP & 93.80 & 92.40 & 93.40 & 94.60 & 94.80 & 94.40 & 91.40 \\
HumanEval & 96.75 & 96.34 & 98.78 & 99.39 & 97.56 & 96.95 & 97.36 \\
GPQA & 69.02 & 69.70 & 73.74 & 70.71 & 72.73 & 70.20 & 66.50 \\
IFEval & 91.44 & 89.65 & 91.68 & 91.68 & 92.98 & 92.79 & 91.13 \\
MMLU & 86.90 & 86.90 & 88.10 & 87.90 & 87.30 & 88.30 & 87.17 \\
LCB C++ & 63.44 & 63.66 & 64.54 & 63.88 & 61.45 & 64.98 & 59.69 \\
AIME 2024 & 71.11 & 76.67 & 66.67 & 80.00 & 83.33 & 73.33 & 65.56 \\
AIME 2025 & 63.33 & 60.00 & 70.00 & 70.00 & 53.33 & 66.67 & 55.56 \\
\midrule
Average & 82.70 & 82.69 & 83.93 & 84.95 & 83.55 & 84.11 & 80.29 \\
\bottomrule
\end{tabular}
\end{table}

\begin{table}[H]
\centering
\caption{\textbf{Per-benchmark wall TPS for the ablations} of Table~\ref{tab:promotion-ablation}. \retfull{} is \OURMETHOD. Best in bold.}
\label{tab:appx-ablation-wall}
\scriptsize
\setlength{\tabcolsep}{4pt}
\begin{tabular}{@{}lrrrrrrr@{}}
\toprule
Benchmark & \retfull{} & \retboth{} & \retsignal{} & \retnoaccept{} & \rettoken{} & \retnosignal{} & \ocabidirectional{} \\
\midrule
GSM8K & \textbf{2624} & 2484 & 2371 & 2344 & 2294 & 1996 & 2037 \\
MATH-500 & \textbf{2605} & 2468 & 2341 & 2367 & 2310 & 2013 & 1853 \\
MBPP & \textbf{2454} & 2309 & 2063 & 2146 & 2206 & 1702 & 1698 \\
HumanEval & \textbf{3338} & 3141 & 3136 & 3099 & 2835 & 2526 & 2533 \\
GPQA & \textbf{2216} & 2118 & 1890 & 1816 & 1888 & 1293 & 1437 \\
IFEval & \textbf{2683} & 2479 & 2268 & 2329 & 2293 & 1835 & 2072 \\
MMLU & \textbf{2351} & 2021 & 1941 & 2068 & 1989 & 1515 & 1808 \\
LCB C++ & \textbf{2026} & 1920 & 1723 & 1749 & 1785 & 1542 & 1126 \\
AIME 2024 & \textbf{2199} & 2123 & 1968 & 1950 & 1821 & 1618 & 1498 \\
AIME 2025 & \textbf{2213} & 2009 & 2056 & 1910 & 1798 & 1651 & 1271 \\
\midrule
Average & \textbf{2471} & 2307 & 2176 & 2178 & 2122 & 1769 & 1733 \\
\bottomrule
\end{tabular}
\end{table}

\begin{table}[H]
\centering
\caption{\textbf{Per-benchmark $\alpha$ and $r$ for the main comparison.} $\alpha$ is undefined for \baselg{}, whose blocks are twice as long as those of \basesm{}.}
\label{tab:appx-alpha-r}
\scriptsize
\setlength{\tabcolsep}{5pt}
\begin{tabular}{@{}lrrrrr@{}}
\toprule
& \multicolumn{2}{c}{$\alpha$} & \multicolumn{3}{c}{$r$} \\
\cmidrule(lr){2-3}\cmidrule(lr){4-6}
Benchmark & \flowblockw{} & \OURMETHOD & \baselg{} & \flowblockw{} & \OURMETHOD \\
\midrule
GSM8K & 0.09 & 0.52 & 1.46 & 1.04 & 1.26 \\
MATH-500 & 0.02 & 0.37 & 1.75 & 1.04 & 1.26 \\
MBPP & 0.03 & 0.45 & 2.01 & 1.02 & 1.27 \\
HumanEval & 0.22 & 0.73 & 1.67 & 1.02 & 1.24 \\
GPQA & 0.10 & 0.63 & 1.95 & 1.03 & 1.26 \\
IFEval & 0.12 & 0.63 & 1.55 & 1.05 & 1.26 \\
MMLU & 0.09 & 0.77 & 1.64 & 1.04 & 1.25 \\
LCB C++ & 0.00 & 0.39 & 2.22 & 1.02 & 1.29 \\
AIME 2024 & 0.07 & 0.42 & 2.00 & 1.03 & 1.25 \\
AIME 2025 & 0.16 & 0.49 & 2.17 & 1.03 & 1.27 \\
\midrule
Average & 0.09 & 0.54 & 1.84 & 1.03 & 1.26 \\
\bottomrule
\end{tabular}
\end{table}

\begin{table}[H]
\centering
\caption{\textbf{Per-benchmark $\alpha$ for the ablations} of Table~\ref{tab:promotion-ablation}. \retfull{} is \OURMETHOD. The fit of Eq.~\eqref{eq:alpha-fit} is solved on $[-\tfrac12,1)$ and so admits $\alpha<0$, meaning a promoted block needs more denoising passes than one initialized from noise. The negative values all occur under \retnosignal{}, are at most $0.09$ in magnitude, and average to $\alpha=0.02$. They are consistent with a warm-up that saves no passes, within the approximation of the block count in Eq.~\eqref{eq:alpha-counts}, and possibly with tokens drafted without self-conditioning against an uncommitted predecessor costing slightly more to revise than noise.}
\label{tab:appx-ablation-alpha}
\scriptsize
\setlength{\tabcolsep}{4pt}
\begin{tabular}{@{}lrrrrrrr@{}}
\toprule
Benchmark & \retfull{} & \retboth{} & \retsignal{} & \retnoaccept{} & \rettoken{} & \retnosignal{} & \ocabidirectional{} \\
\midrule
GSM8K & 0.52 & 0.39 & 0.30 & 0.27 & 0.21 & $-0.01$ & 0.40 \\
MATH-500 & 0.37 & 0.28 & 0.20 & 0.21 & 0.15 & $-0.02$ & 0.22 \\
MBPP & 0.45 & 0.35 & 0.18 & 0.23 & 0.24 & $-0.07$ & 0.34 \\
HumanEval & 0.73 & 0.61 & 0.60 & 0.58 & 0.39 & 0.21 & 0.52 \\
GPQA & 0.63 & 0.54 & 0.37 & 0.31 & 0.33 & $-0.09$ & 0.44 \\
IFEval & 0.63 & 0.48 & 0.33 & 0.37 & 0.33 & 0.02 & 0.47 \\
MMLU & 0.77 & 0.47 & 0.40 & 0.50 & 0.41 & 0.02 & 0.66 \\
LCB C++ & 0.39 & 0.30 & 0.16 & 0.18 & 0.17 & 0.00 & 0.22 \\
AIME 2024 & 0.42 & 0.36 & 0.26 & 0.24 & 0.14 & 0.01 & 0.28 \\
AIME 2025 & 0.49 & 0.35 & 0.36 & 0.28 & 0.16 & 0.09 & 0.25 \\
\midrule
Average & 0.54 & 0.41 & 0.32 & 0.32 & 0.26 & 0.02 & 0.38 \\
\bottomrule
\end{tabular}
\end{table}

\begin{table}[H]
\centering
\caption{\textbf{Per-benchmark $r$ for the ablations} of Table~\ref{tab:promotion-ablation}. \retfull{} is \OURMETHOD.}
\label{tab:appx-ablation-r}
\scriptsize
\setlength{\tabcolsep}{4pt}
\begin{tabular}{@{}lrrrrrrr@{}}
\toprule
Benchmark & \retfull{} & \retboth{} & \retsignal{} & \retnoaccept{} & \rettoken{} & \retnosignal{} & \ocabidirectional{} \\
\midrule
GSM8K & 1.26 & 1.26 & 1.26 & 1.26 & 1.26 & 1.27 & 1.49 \\
MATH-500 & 1.26 & 1.26 & 1.27 & 1.26 & 1.26 & 1.27 & 1.64 \\
MBPP & 1.27 & 1.27 & 1.28 & 1.27 & 1.25 & 1.29 & 1.76 \\
HumanEval & 1.24 & 1.25 & 1.25 & 1.25 & 1.25 & 1.25 & 1.50 \\
GPQA & 1.26 & 1.25 & 1.27 & 1.27 & 1.25 & 1.29 & 1.79 \\
IFEval & 1.26 & 1.26 & 1.26 & 1.26 & 1.26 & 1.26 & 1.51 \\
MMLU & 1.25 & 1.26 & 1.27 & 1.25 & 1.25 & 1.26 & 1.53 \\
LCB C++ & 1.29 & 1.28 & 1.29 & 1.29 & 1.27 & 1.28 & 2.14 \\
AIME 2024 & 1.25 & 1.25 & 1.27 & 1.26 & 1.26 & 1.27 & 1.71 \\
AIME 2025 & 1.27 & 1.28 & 1.26 & 1.28 & 1.26 & 1.30 & 1.94 \\
\midrule
Average & 1.26 & 1.26 & 1.27 & 1.27 & 1.26 & 1.27 & 1.70 \\
\bottomrule
\end{tabular}
\end{table}
\end{document}